\documentclass[10pt,twocolumn,letterpaper]{article}

\usepackage{cvpr}
\usepackage{times}
\usepackage{epsfig}
\usepackage{graphicx}
\usepackage{subcaption}
\usepackage{amsmath}
\usepackage{amssymb}

\usepackage{amsthm}
\usepackage[ruled,vlined]{algorithm2e}
\usepackage{makecell}

\DeclareMathOperator*{\argmin}{argmin}

\usepackage{xcolor,colortbl}
\definecolor{Gray}{gray}{0.9}

\usepackage[pagebackref=true,breaklinks=true,letterpaper=true,colorlinks,bookmarks=false]{hyperref}

\cvprfinalcopy 
\def\cvprPaperID{1929}

\ifcvprfinal\fi
\begin{document}

\renewcommand{\thesubfigure}{\Alph{subfigure}}

\newcommand{\norm}[1]{\left\lVert#1\right\rVert_2}

\newtheorem{proposition}{Proposition}
\newtheorem{corollary}{Corollary}
\newtheorem{theorem}{Theorem}

\makeatletter
\renewenvironment{proof}[1][\proofname]{\par
  \vspace{-\topsep}
  \pushQED{\qed}%
  \normalfont
  \topsep0pt \partopsep0pt 
  \trivlist
  \item[\hskip\labelsep
        \itshape
    #1\@addpunct{.}]\ignorespaces
}{%
  \popQED\endtrivlist\@endpefalse
  \addvspace{6pt plus 6pt} 
}
\makeatother

\title{
Adversarial Structure Matching for Structured Prediction Tasks 
}

\author{
\begin{tabular}{c@{\hskip 20pt}c@{\hskip 20pt}c@{\hskip 20pt}c}
    Jyh-Jing Hwang$^{1,2,*}$ &
    Tsung-Wei Ke$^{1,*}$ &
    Jianbo Shi$^2$ & 
    Stella X. Yu$^1$
\end{tabular} \\
\begin{tabular}{c@{\hskip 20pt}c}
    $^1$UC Berkeley / ICSI & $^2$University of Pennsylvania \\
    Berkeley, CA, USA & Philadelphia, PA, USA \\
    {\tt\small \{jyh,twke,stellayu\}@berkeley.edu} & {\tt\small \{jyh,jshi\}@seas.upenn.edu}
\end{tabular}
}

\maketitle
\thispagestyle{empty}

{\let\thefootnote\relax\footnote{* Equal contributions.}}

\begin{abstract}
Pixel-wise losses, e.g., cross-entropy or L2, have been widely used in structured prediction tasks as a spatial extension of generic image classification or regression. However, its i.i.d. assumption neglects the structural regularity present in natural images.
Various attempts have been made to incorporate structural reasoning mostly through structure priors in a cooperative way where co-occurring patterns are encouraged. 

We, on the other hand, approach this problem from an opposing angle and propose a new framework, Adversarial Structure Matching (ASM), for training such structured prediction networks via an adversarial process, in which we train a structure analyzer that provides the supervisory signals, the ASM loss.
The structure analyzer is trained to maximize the ASM loss, or to emphasize recurring multi-scale hard negative structural mistakes among co-occurring patterns.
On the contrary, the structured prediction network is trained to reduce those mistakes and is thus enabled to distinguish fine-grained structures.
As a result, training structured prediction networks using ASM reduces contextual confusion among objects and improves boundary localization.
We demonstrate that our ASM outperforms pixel-wise IID loss or structural prior GAN loss on three different structured prediction tasks: semantic segmentation, monocular depth estimation, and surface normal prediction.

\end{abstract}

\section{Introduction}
\label{sec:intro}

Pixel-wise losses, e.g. cross-entropy or L2, are widely used in structured prediction tasks such as semantic segmentation, monocular depth estimation, and surface normal prediction \cite{hwang2015pixel,long2015fully,eigen2014depth,laina2016deeper}, as a spatial extension of generic image recognition~\cite{krizhevsky2012imagenet,he2016deep}. However, the disadvantage of such pixel-wise losses is also obvious due to its additive nature and i.i.d. assumption of predictions: IID losses would yield the same overall error for different spatial distributions of prediction mistakes. Ideally, some mistakes such as implausible and incomplete round wheels should incur more penalty than slightly thinner wheels.  Structural reasoning is thus highly desirable for structured prediction tasks.

\begin{figure}[t]
    \centering
    \includegraphics[width=1.0\linewidth]{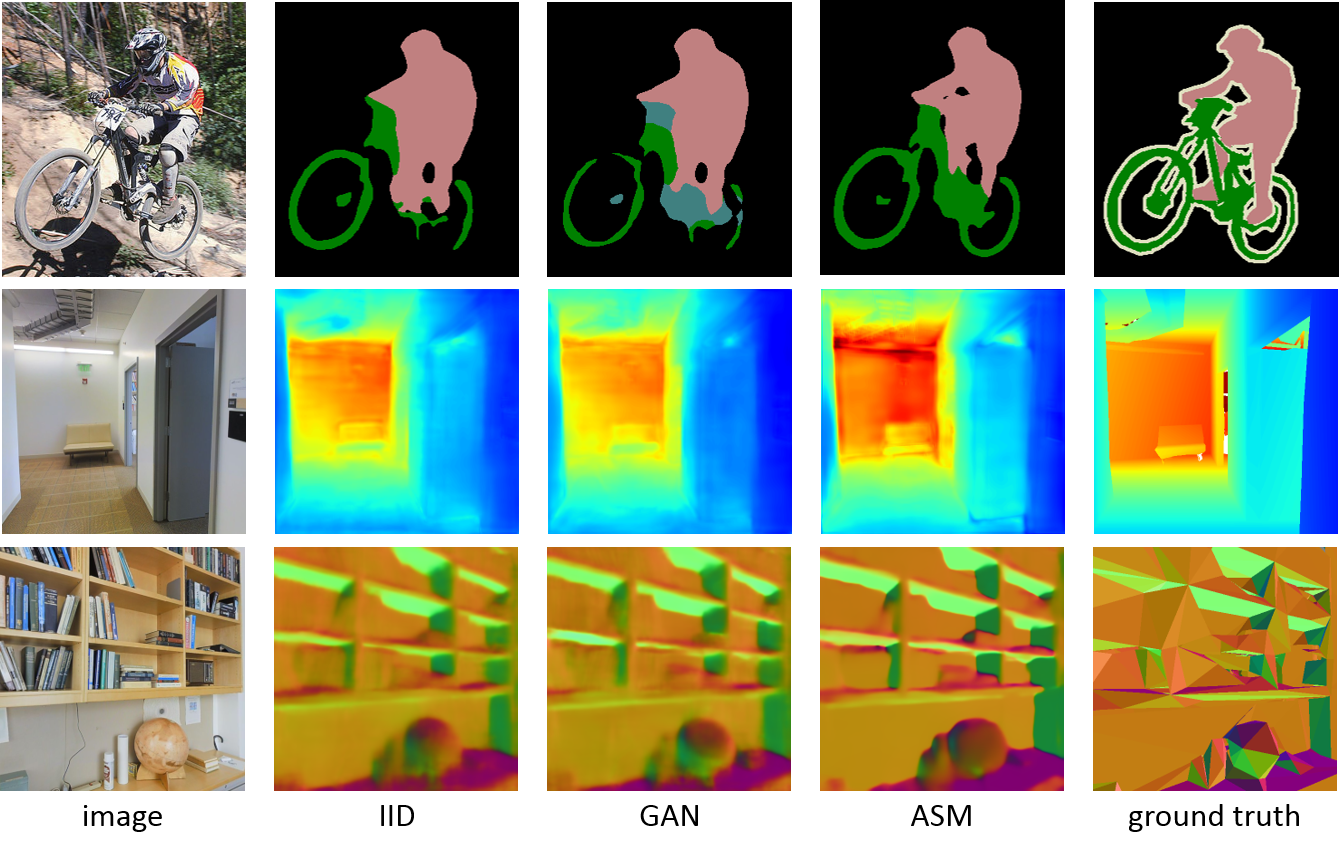}
    \caption{Experimental results on semantic segmentation on PASCAL VOC 2012~\cite{everingham2010pascal}, monocular depth estimation, and surface normal prediction on Stanford 2D-3D-S \cite{2D-3D-S}. Networks trained using IID loss and shape prior based (GAN) loss mostly fail at confusing contexts (top row) and ambiguous boundaries (bottom 2 rows) whereas our ASM approach improves upon these aspects.}
    \label{fig:teaser}
\end{figure}

Various attempts have been made to incorporate structural reasoning into structured prediction in a \textit{cooperative} way, including two mainstreams, 
bottom-up Conditional Random Fields (CRFs)~\cite{krahenbuhl2011efficient,zheng2015conditional} and top-down shape priors~\cite{xie2016top,gygli2017value,ke2018adaptive} or Generative Adversarial Networks (GANs)~\cite{goodfellow2014generative,luc2016semantic,pix2pix2017}:
{\bf (1)} CRF enforces label consistency between pixels and is commonly employed as a post-processing step~\cite{krahenbuhl2011efficient,chen2016deeplab}, or as a plug-in module inside deep neural networks~\cite{zheng2015conditional,liu2015semantic} that coordinate bottom-up information.
Effective as it is, CRF is usually sensitive to input appearance changes and needs expensive iterative inference.
{\bf (2)} As an example of learning top-down shape priors, GAN emerges as an alternative to enforce structural regularity in the structured prediction space. Specifically, the discriminator network is trained to distinguish the predicted mask from the ground truth mask. Promising as it is, GAN suffers from inaccurate boundary localization as a consequence of generic shape modeling.
More recently, Ke~{\it et~al.}~\cite{ke2018adaptive} propose adaptive affinity fields that capture structural information with adaptive receptive fields. However, it is designed specifically for classification and cannot be extended straightforward to regression.

As a result, top-down \textit{cooperative} approaches prefer an additional loss (together with IID) that penalizes more on the anomaly structures that are deemed undesirable. Such trained networks are thus aware of intra-category shape invariance and inter-category object co-occurrences. However, we notice that in real examples as in Fig.~\ref{fig:teaser}, complex and deformable shapes and confusing co-occurrences are the most common mistakes in structured prediction especially when the visual cues are ambiguous. As a result, training with shape priors sometimes deteriorates the prediction as shown in the bicycle example. We are thus inspired to tackle this problem from an \textit{opposing} angle: top-down approaches should adapt the focus to confusing co-occurring context or ambiguous boundaries so as to make the structured prediction network learn harder.

We propose a new framework called Adversarial Structure Matching (ASM), which replaces IID losses, for training structured prediction networks via an adversarial process, in which we train a structure analyzer to provide supervisory signals, the adversarial structure matching (ASM) loss. By maximizing ASM loss, or learning to exaggerate structural mistakes from the structured prediction networks, the structure analyzer not only becomes aware of complex shapes of objects but adaptively emphasize those {\it multi-scale hard negative structural mistakes}. As a result, training structured prediction networks by minimizing ASM loss reduces contextual confusion among co-occurring objects and improves boundary localization. To improve the stability of training, we append a structure regularizer on the structure analyzer to compose a structure autoencoder. By training the autoencoder to reconstruct ground truth, which contains complete structures, we ensure the filters in the structure analyzer form a good structure basis. We demonstrate that structured prediction networks trained using ASM outperforms its pixel-wise counterpart and GAN on the figure-ground segmentation task on Weizmann horse dataset~\cite{borenstein2002horse} and semantic segmentation task on PASCAL VOC 2012 dataset~\cite{everingham2010pascal} with various base architectures, such as FCN~\cite{long2015fully}, U-Net~\cite{ronneberger2015u}, DeepLab~\cite{chen2016deeplab}, and PSPNet~\cite{zhao2016pyramid}. Besides the structured classification tasks, we also verify the efficacy of ASM on structured regression tasks, such as monocular depth estimation and surface normal prediction, with U-Net \cite{ronneberger2015u} on Stanford 2D-3D-S dataset \cite{2D-3D-S}.

\section{Related Works}
\label{sec:work}

\begin{figure*}[t]
    \centering
    \includegraphics[width=0.95\linewidth]{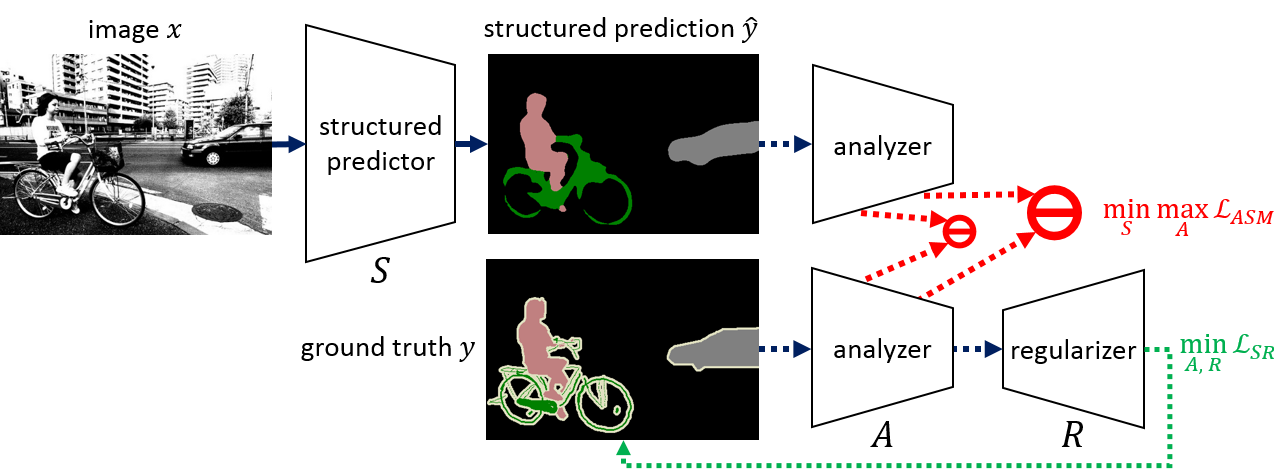}
    \caption{Framework overview: A structure analyzer extracts structure features (red arrows) from structured outputs. The analyzer is trained to maximize an adversarial structure matching (ASM) loss, or discrepancy between multi-scale structure features extracted from ground truth and from predictions of a structured predictor. The analyzer thus learns to exaggerate the hard negative structural mistakes and to distinguish fine-grained structures. The structured predictor on the contrary is trained to minimize the ASM loss. To make sure the filters in the analyzer form a good structure basis, we introduce a structure regularizer, which together with the analyzer, constitutes an autoencoder. The autoencoder is trained to reconstruct ground truth with a structure regularization (SR) loss. Dotted lines denote computations during training only.}
    \label{fig:architecture}
\end{figure*}

\noindent
\textbf{Semantic Segmentation.}
%
%
The field of semantic segmentation has progressed fast in the last few years since the introduction of fully convolutional networks~\cite{long2015fully,chen2016deeplab}. Both deeper~\cite{zhao2016pyramid,li2017not} and wider~\cite{noh2015learning,ronneberger2015u,yu2015multi} network architectures have been proposed and have dramatically boosted the performance on standard benchmarks like PASCAL VOC 2012~\cite{everingham2010pascal}. Notably, multi-scale context information emerges to remedy limited receptive fields, e.g., spatial pyramid pooling~\cite{zhao2016pyramid} and atrous spatial pyramid pooling~\cite{chen2017rethinking}.
Though these methods yield impressive performance w.r.t. mIoU, they fail to capture rich structure information present in natural scenes.

\vspace{6pt}

\noindent
\textbf{Structure Modeling.}
To overcome the aforementioned drawback, researchers have explored several ways to incorporate structure information~\cite{krahenbuhl2011efficient,chen2015learning,zheng2015conditional,liu2015semantic,lin2016efficient,bertasius2016convolutional,xie2016top,gygli2017value,ke2018adaptive,MMS:CVPR:2018}.
For example, Chen et al.~\cite{chen2016deeplab} utilized denseCRF~\cite{krahenbuhl2011efficient} as post-processing to refine the final segmentation results.
Zheng et al.~\cite{zheng2015conditional} and Liu et al.~\cite{liu2015semantic} further made the CRF module differentiable within the deep neural network.
Besides, low-level cues, such as affinity~\cite{shi2000normalized,maire2016affinity,liu2017learning,bertasius2016convolutional} and contour~\cite{bertasius2016semantic,chen2016semantic} have also been leveraged to encode image structures.
However, these methods either are sensitive to appearance changes or require expensive iterative inference.

\vspace{6pt}

\noindent
\textbf{Monocular Depth Estimation and Normal Prediction.}
With large-scale RGB-D data available \cite{silberman2012indoor,2D-3D-S,zamir2018taskonomy}, data-driven approaches \cite{karsch2014depth,kendall2017uncertainties,ladicky2014pulling,laina2016deeper,li2017two,liu2016learning,roy2016monocular,eigen2014depth,fouhey2013data} based on deep neural networks make remarkable progress on depth estimation and surface normal prediction. Just like in semantic segmentation, some incorporate structural regularization such as CRFs into the system and demonstrate notable improvements \cite{kim2016unified,li2015depth,liu2015deep,wang2015towards,wang2016surge,xu2017multi,zhuo2015indoor,cheng2018depth}. Recently, leveraging depth estimation and semantic segmentation for each other has emerged to be a promising direction to improve structured predictions of both tasks \cite{couprie2013indoor,eigen2015predicting,gupta2014learning,ren2012rgb,silberman2012indoor}.
Others propose to jointly train a network for both tasks \cite{misra2016cross,richter2017playing,kokkinos2017ubernet,zhang2018joint,xu2018pad}. Our work is also along this direction as the ASM framework can train different structured prediction tasks in a consistent manner.

\section{Method}
\label{sec:method}

We provide an overview of the framework in Fig.~\ref{fig:architecture} and summarize the training procedure in Alg.~\ref{alg}.

\subsection{Adversarial Structure Matching}
\label{sec:asml}

We consider structured prediction tasks, in which a structured prediction network (structured predictor) $S: {\bf x} \mapsto {\bf \hat{y}}$, which usually is a deep CNN, is trained to map an input image ${\bf x} \in \mathbb{R}^n$ to a per-pixel label mask ${\bf \hat{y}} \in \mathbb{R}^n$.
We propose to train such a structured predictor with another network, structure analyzer. The analyzer $A:\mathbb{R}^n\mapsto \mathbb{R}^k$ extracts $k$-dimensional multi-layer structure features from either ground truth masks, denoted as $A({\bf y})$, or predictions, denoted as $A(S({\bf x}))$. We train the analyzer to maximize the distance between the structure features from either inputs, so that it learns to exaggerate structural mistakes, or hard negative structural examples, made by the structured predictor. On the contrary, we simultaneously train the structured predictor to minimize the same distance. In other words, structured predictor $S$ and structure analyzer $A$ play the following two-player minimax game with value function $V(S, A)$:
\begin{equation}
    \min_S \max_A V(S,A) = \mathbb{E}_{\bf x, y} \left[\frac{1}{2} \norm{A\left(S({\bf x})\right) - A({\bf y})}^2 \right],
\end{equation}
that is, we prefer the optimal structured predictor as the one that learns to predict true structures to satisfy the analyzer. Note that the analyzer will bias its discriminative power towards similar but subtly different structures as they occur more frequently through the course of training.

One might relate this framework to GAN~\cite{goodfellow2014generative}. A critical distinction is that GAN tries to minimize the data distributions between real and fake examples and thus accepts a set of solutions. Here, structured prediction tasks require a specific one-to-one mapping of each pixel between ground truth masks and predictions. Therefore, the discrimination of structures should take place for every patch between corresponding masks, hence the name adversarial structure matching (ASM).

It is also related to perceptual loss~\cite{gatys2015neural,zhang2018unreasonable} for style transfer, which uses pretrained CNN features to capture image statistics. We, on the other hand, generalize this concept by adapting the CNN and accepting any dimensional inputs.

\subsection{Global Optimality of $S({\bf x})={\bf y}$ and Convergence}
\label{sec:thoery}

We would like the structured predictor to converge to a good mapping of ${\bf y}$ given ${\bf x}$, if given enough capacity and training time. To simplify the dynamic of convergence, we consider both structured predictor $S$ and analyzer $A$ as models with infinite capacity in a non-parametric setting.

\begin{proposition}
For a fixed $S$, if $S({\bf x})\neq{\bf y}$, then $\norm{A^*\left(S({\bf x})\right) - A^*({\bf y})}^2$ is infinitely large for an optimal $A$.
\end{proposition}
\begin{proof}
If $S({\bf x})\neq{\bf y}$, there exists an index $i$ such that $S({\bf x})[i]-y[i] = \epsilon$, where $\epsilon \in \mathbb{R} \setminus \{0\}$. Without loss of generality, we assume $S({\bf x})[j] = y[j]$ if $j\neq i$ and let $S({\bf x})[i] = c + \frac{1}{2}\epsilon$ and $y[i] = c - \frac{1}{2}\epsilon$.

We consider a special case where $A_l$ on the $i$-th dimension of the input is a linear mapping, i.e., $A_l(x[i])=w_ix[i]$. As $A$ is with infinite capacity, we know there exists $A$ such that
\begin{align}
\nonumber \norm{A\left(S({\bf x})\right) - A({\bf y})}^2 \ge
    \norm{A_l\left(S({\bf x})\right) - A_l({\bf y})}^2 = \\
    \norm{w_i\left(c + \frac{1}{2}\epsilon\right) - w_i\left(c - \frac{1}{2}\epsilon\right)}^2=|w_i\epsilon|
\end{align}
Note that $\norm{A_l\left(S({\bf x})\right) - A_l({\bf y})}^2 \to \infty$ as $|w_i|\to \infty$. Thus $\norm{A^*\left(S({\bf x})\right) - A^*({\bf y})}^2\to \infty$.
\end{proof}
In practice, parameters of $A$ are restricted within certain range under weight regularization so $\norm{A^*\left(S({\bf x})\right) - A^*({\bf y})}^2$ would not go to infinity.

\begin{corollary}
For an optimal $A$, $S({\bf x})={\bf y}$ if and only if $A^*(S({\bf x}))=A^*({\bf y})$.
\end{corollary}
\begin{proof}
$\Rightarrow$ If $S({\bf x})={\bf y}$, $\norm{A\left(S({\bf x})\right) - A({\bf y})}^2 = \norm{A({\bf y})-A({\bf y})}^2 = 0$, for any $A$. Hence $\norm{A^*\left(S({\bf x})\right) - A^*({\bf y})}^2 = 0$. \\
$\Leftarrow$ If $A^*(S({\bf x}))=A^*({\bf y})$ or $\norm{A^*\left(S({\bf x})\right) - A^*({\bf y})}^2=0$, $S({\bf x})\neq{\bf y}$ contradicts Proposition 1. Hence $S({\bf x})={\bf y}$.
\end{proof}

\begin{theorem}
If ($S^*, A^*$) is a Nash equilibrium of the system, then $S^*({\bf x})={\bf y}$ and $V(S^*,A^*)=0$
\end{theorem}
\begin{proof}
From Proposition 1, we proved $V(S,A^*)\to \infty$ if $S({\bf x})\neq{\bf y}$. From Corollary 1, we proved $V(S,A^*)=0$ if and only if $S({\bf x})={\bf y}$. Since $V(S,A)\ge 0$ for any $S$ and $A$, the Nash equilibrium only exists when $S^*({\bf x})={\bf y}$, or $V(S^*,A^*)=0$.
\end{proof}

From the proofs, we recognize imbalanced powers between the structured predictor and structure analyzer where the analyzer can arbitrarily enlarge the value function if the structured predictor is not optimal. In practice, we should limit the training of the analyzers or apply regularization, such as weight regularization or gradient capping to prevent gradient exploding. Therefore, we train the analyzer only once per iteration with a learning rate that is not larger than the one for structured predictor. Another trick for semantic segmentation is to binarize the predictions $S(x)$ (winner-take-all across channels for every pixel) before calculating ASM loss for analyzer. In this way, the analyzers will focus on learning to distinguish the structures instead of the confidence levels of predictions.

\begin{algorithm}[t]
\SetAlgoLined
\SetNoFillComment
\For{number of training iterations} {
    \tcc{Train structure analyzer}
    
    (optional) Binarize structured predictions $S({\bf x}^{(i)})$. \\
    Update the structure analyzer $A$ and regularizer $R$ by ascending its stochastic gradient:
    \begin{multline}
    \nonumber \nabla_{\theta_{A,R}}\frac{1}{m}\sum_{i=1}^m \Big[\frac{1}{2}\norm{ A\left(S({\bf x}^{(i)}) \right) - A({\bf y}^{(i)}) }^2\\
    \nonumber + \lambda IID\left({\bf y}, R\left(A_t({\bf y})\right)\right) \Big]
    \end{multline}
    
    \tcc{Train structured predictor}
    Sample a minibatch with $m$ images $\{ {\bf x}^{(1)}, \dots, {\bf x}^{(m)} \}$ and ground truth masks $\{ {\bf y}^{(1)}, \dots, {\bf y}^{(m)} \}$. \\
    Update the structured predictor $S$ by descending its stochastic gradient:
    $$\nabla_{\theta_S}\frac{1}{2m}\sum_{i=1}^m \norm{ A\left(S({\bf x}^{(i)})\right)-A({\bf y}^{(i)}) }^2 $$
    }
The gradient-based updates can use any standard gradient-based learning rule. Structure analyzer $A$ should use a learning rate that is not larger than structured predictor $S$.
\caption{Algorithm for training structured prediction networks using ASM.}
\label{alg}
\end{algorithm}

\subsection{Reconstructing ${\bf y}$ as Structure Regularization}
\label{sec:regularization}

Although theoretically structure analyzers would discover any structural difference between predictions and ground truth, randomly initialized analyzers suffer from missing certain structures in the early stage. For example, if filter responses for a sharp curve are initially very low, ASM loss for the sharp curve will be as small, resulting in inefficient learning. This problem will emerge when training both structured predictors and structure analyzers from scratch. To alleviate this problem, we propose a regularization method to stabilize the learning of analyzers.

One way to ensure the filters in the analyzer form a good structure basis is through reconstructing ground truth, which contains complete structures. If filters in the analyzer fail to capture certain structures, the ground truth mask cannot be reconstructed. Hence, we append a structure regularizer on top of structure analyzer to constitute an autoencoder. we denote the structure regularizer $R:A_t({\bf y}) \mapsto {\bf y}$, where $A_t(\cdot)$ denotes features from the structure analyzer, which are not necessarily the same set of features for ASM; hence the reconstruction mapping: $R\left(A_t({\bf y})\right) \mapsto {\bf y}$. As a result, the final objective function is as follows
\begin{align}
\nonumber  S^* =& \underbrace{\argmin_S  \max_A \mathbb{E}_{\bf x, y} \left[\frac{1}{2}\norm{A\left(S({\bf x})\right) - A({\bf y})}^2\right]}_\text{adversarial structure matching loss} \\
+& \underbrace{\min_{A, R} \lambda \mathbb{E}_{\bf y} \left[IID\left({\bf y}, R\left(A_t({\bf y})\right)\right) \right]}_\text{structure regularization loss},
\end{align}
where $IID\left({\bf y}, R\left(A_t({\bf y})\right)\right)$ is defined for target tasks as: 
\begin{align*}
&\text{Semantic segmentation:} -{\bf y} \cdot \log R\left(A_t({\bf y})\right) \\
&\text{Depth estimation:} \norm{y-R\left(A_t({\bf y})\right)}^2 \\
&\text{Surface normal Prediction:} \norm{\frac{y}{\norm{y}}-\frac{R\left(A_t({\bf y})\right)}{\norm{R\left(A_t({\bf y})\right)}}}^2
\end{align*}
Note that structure regularization loss is independent to $S$.

\begin{figure*}
    \centering
    \includegraphics[width=1.0\linewidth]{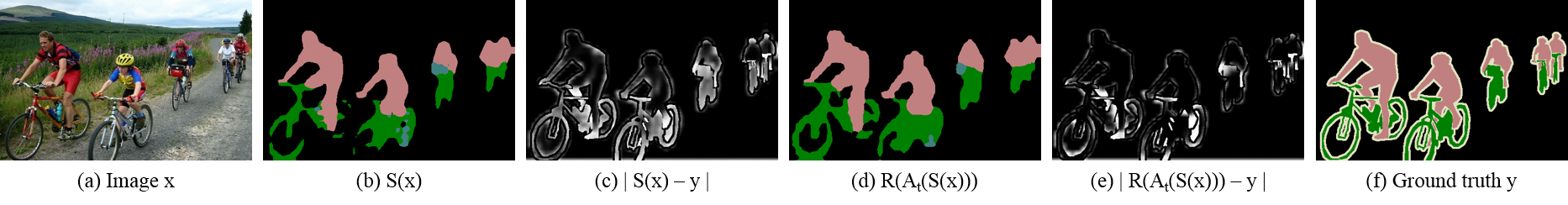}
    \caption{Error maps of predictions and reconstructions by the structure autoencoder on VOC PASCAL 2012. The autoencoder network is able to reconstruct the missing part of certain structures, for example, it completes the round wheel in the front. Note that neither the structure analyzer nor regularizer has access to input images.}
    \label{fig:seg_errors}
\end{figure*}

\begin{figure*}
    \centering
    \includegraphics[width=1.0\linewidth]{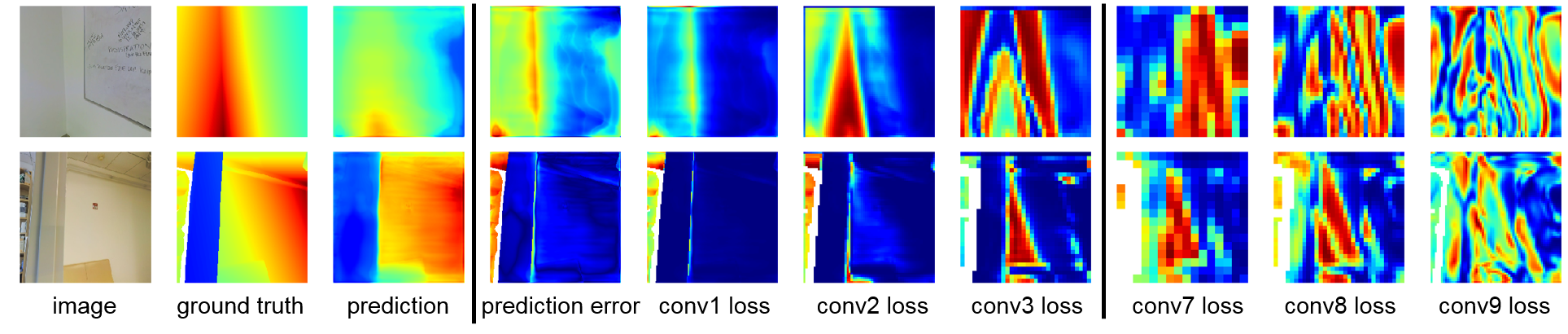}
    \caption{Visualization of loss maps from different layers in the structure analyzer. Examples are from monocular depth estimation on 2D-3D-Semantics~\cite{2D-3D-S}. We observe earlier layers capture lower-level structural mistakes as opposed to later layers. It is worth to note that $conv9$, with skip connection from $conv1$, provides low-level structure errors guided by global information. We observe that multi-scale structure supervision is achieved by the multi-layer structure analyzer. }
    \label{fig:analyzer_analysis}
\end{figure*}

\section{Experiments}
\label{sec:exp}

We demonstrate the effectiveness of our proposed methods on three structure modeling tasks: semantic segmentation, monocular depth estimation, and surface normal prediction. We conduct extensive experiments on all three tasks to compare the same structured prediction networks trained using ASM, GAN, or IID losses. (Note that GAN means training with IID loss and adversarial loss.)

We first give an overview of the datasets and implementation details in Sec. \ref{sec:exp_setup}. Then we analyze what is learned in the structure analyzer to understand the mechanics of ASM in Sec. \ref{sec:exp_analysis}. Finally, we present our main results and analyses on segmentation in Sec. \ref{sec:exp_seg} and on depth estimation and surface normal prediction in Sec. \ref{sec:exp_depth}.

\subsection{Experimental Setup}
\label{sec:exp_setup}

\noindent
\textbf{Tasks and datasets.}
We compare our proposed ASM against GAN and IID losses on the Weizmann horse~\cite{borenstein2002horse}, PASCAL VOC 2012~\cite{everingham2010pascal}, and Stanford 2D-3D-Semantics \cite{2D-3D-S} datasets. The {\bf Weizmann horse} \cite{borenstein2002horse} is a relatively small dataset for figure-ground segmentation that contains $328$ side-view horse images, which are split into $192$ training and $136$ validation images. The {\bf VOC 2012} \cite{everingham2010pascal} dataset is a well-known benchmark for generic semantic segmentation which includes $20$ object classes and a `background' class, containing $10,582$ and $1,449$ images for training and validation, respectively. We also conduct experiments on {\bf 2D-3D-Semantics}~\cite{2D-3D-S} dataset for monocular depth and surface normal prediction. The 2D-3D-S dataset is a large-scale indoor scene benchmark, which consists of $70,496$ images, along with depth, surface normal, instance- and semantic-level segmentation annotations. The results are reported over fold-1 data splits--Area 1,2,3,4,6 ($52,903$ images) for training and Area 5 ($17,593$ images) for testing.

\vspace{6pt}
\noindent
\textbf{Architectures.} For all the structure autoencoders (i.e., analyzer and regularizer), we use U-Net~\cite{ronneberger2015u} with either 7 conv layers for figure-ground segmentation, 5 conv layers for semantic segmentation and 9 conv layers for depth and surface normal estimation.
We conduct experiments on different structured prediction architectures. On horse dataset~\cite{borenstein2002horse}, we use U-Net~\cite{ronneberger2015u} (with 7 convolutional layers) as the base architecture. On VOC~\cite{everingham2010pascal} dataset, we carry out experiments and thorough analyses over $3$ different architectures with ResNet-101~\cite{he2016deep} backbone, including FCN~\cite{long2015fully}, DeepLab~\cite{chen2016deeplab}, and PSPNet~\cite{zhao2016pyramid}, which is a highly competitive segmentation model. On 2D-3D-S \cite{2D-3D-S} dataset, we use U-Net \cite{ronneberger2015u} (with 18 convolutional layers) as the base architecture for both depth and surface normal prediction. Aside from base architectures, neither extra parameters nor post-processing are required at inference time.

\vspace{6pt}
\noindent
\textbf{Implementation details on Weizmann horse.}
We use the poly learning rate policy where the current learning rate equals the base one multiplied by  $(1-\frac{\text{iter}}{\text{max\_iter}})^{0.9}$ with max iterations as $100$ epochs. We set the base learning rate as $0.0005$ with Adam optimizer for both $S$ and $A$. Momentum and weight decay are set to $0.9$ and $0.00001$, respectively. We set the batch size as $1$ and use only random mirroring. For ASM, We set $\lambda=2$ for structure regularization.

\vspace{6pt}
\noindent
\textbf{Implementation details on VOC dataset.}
Our implementation follows the implementation details depicted in ~\cite{chen2017rethinking}. We adopt the same poly learning rate policy and set the base learning rate with SGD optimizer as $0.001$  for $S$ and $0.0005$ for $A$. The training iterations for all experiments on all datasets are $30$K. Momentum and weight decay are set to $0.9$ and $0.0005$, respectively.  For data augmentation, we adopt random mirroring and random resizing between 0.5 and 2 for all
datasets. We do not use random rotation and random Gaussian blur. We do not upscale the logits (prediction map) back to the input image resolution, instead, we follow  \cite{chen2016deeplab}'s setting by downsampling the ground-truth labels for training ($output\_stride=8$). The crop size is set to $336 \times 336$ and batch size is set to $8$. We update BatchNorm parameters with $decay=0.9997$ for ImageNet-pretrained layers and $decay=0.99$ for untrained layers. For ASM, we set $\lambda=10$ for structure regularization.

\vspace{6pt}
\noindent
\textbf{Implementation details on 2D-3D-S.}
We implement pixel-wise L2 loss and normalized L2 for depth and surface normal prediction as a counterpart to cross-entropy loss for segmentation. We adopt the same poly learning rate policy, and set the base learning rate to $0.01$ for both $S$ and $A$. While the original image resolution is $1080 \times 1080$, we down-sample image to half resolution, set ``cropsize'' to $512 \times 512$ for training, and keep $540 \times 540$ for testing. We use random cropping and mirroring for data augmentation. The ``batchsize'' is set to $8$, weight decay is set to $0.0005$, and models are trained for $120$K iterations from scratch.

\subsection{Analyses of Structure Analyzer}
\label{sec:exp_analysis}

Before getting to the main results, we verify the structure analyzer actually captures certain structures and visualize what kind of supervisory signals it provides for training the structured predictor.

For segmentation, some objects are rigid and certain shapes are important structures, \eg, rectangles for buses, circles for wheels, and parallel lines for poles. Here, we visualize the prediction, reconstruction, and their error maps in Fig. \ref{fig:seg_errors}. We observe that that the structure autoencoder network is able to reconstruct the missing part of certain structures, for example, it completes round wheels of bikes. Note that neither the structure analyzer nor regularizer has access to input images.

We further analyze the losses from each layer in the structure analyzer for monocular depth estimation, shown in Fig. \ref{fig:analyzer_analysis}. We observe that in the early layers where information is local, the errors are near object boundaries. In the middle layers, \eg, layer 2 and 3, the structure analyzer attends to broader context, such as the orientation of walls. In the final layer ($conv9$, with skip connection from $conv1$), low-level structure errors are guided by global information. We observe that multi-scale structure supervision is achieved by the multi-layer structure analyzer.

To understand what's learned in the analyzer, we are inspired by Zhou {\it et al.}~\cite{zhou2018interpreting} to visualize features in a CNN by finding the patches that maximally activate a targeted filter. We show in Fig.~\ref{fig:rebut_vis} the top 10 input stimuli for three filters in different layers for each task. We observe that each filter in the analyzer attends to a specific structure, {\it e.g.}, chairs, bottles, hallways, bookcases, {\it etc.}.

\begin{figure}
    \centering
    \includegraphics[width=1.0\linewidth]{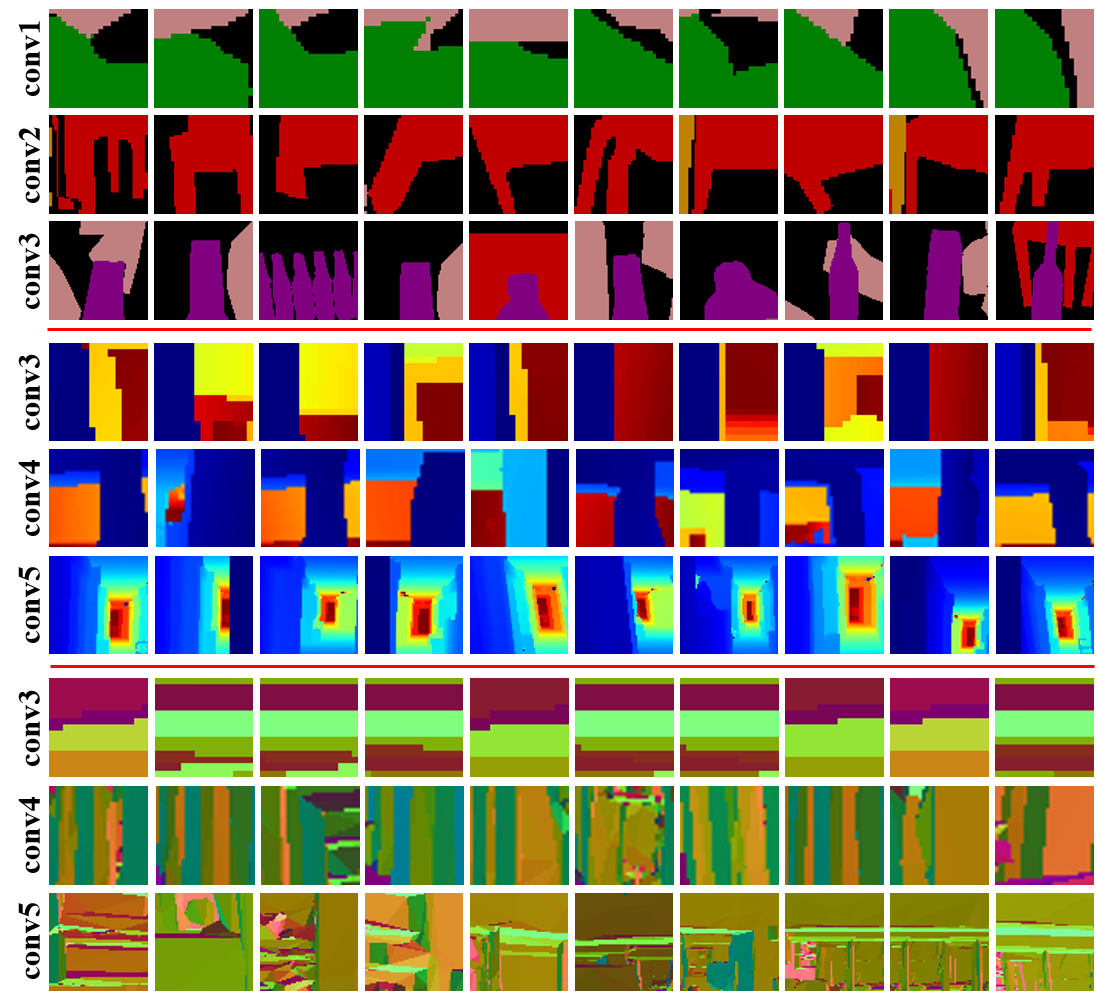}
    \caption{Top 10 input stimuli (in a row) for a filter in the analyzer with maximal activations for semantic segmentation (top 3 rows), depth estimation (middle 3 rows), and surface normal prediction (bottom 3 rows).}
    \label{fig:rebut_vis}
\end{figure}

\subsection{Main Results on Segmentation}
\label{sec:exp_seg}

\begin{table}
\centering
\begin{tabular}{|l||c|c|}
\hline
    \textbf{Loss}  & \textbf{Horse mIoU (\%)} & \textbf{VOC mIoU (\%)} \\ \hline 
    IID & 77.28 & 68.91\\ \hline \hline
    ASM (Conv1) & 78.14 & 70.00 \\ \hline 
    ASM (Conv2) & 78.15 & 70.70 \\ \hline
    ASM (Conv1-2) & \textbf{79.62} & 71.60 \\ \hline
    ASM (Conv3) & 77.79 & 70.85 \\ \hline
    ASM (Conv1-3) & 78.11 & 69.81 \\ \hline \hline
    ASM w/o rec. & 77.83 & \textbf{72.14} \\ \hline
    ASM w/o adv. & 76.70 & 71.26 \\ \hline \hline
    IID+ASM & 78.34 & 68.49\\ \hline
\end{tabular}
\caption{Ablation studies of ASM on Weizmann horse dataset with U-Net and on PASCAL VOC dataset with FCN. Generally, using low- to mid-level features (conv1 and conv2) of structure analyzers yield the best performance. It also shows that reconstruction is not always needed if base networks are pre-trained. DeepLab or PSPNet (not shown here) has the same trend as U-Net.}
\label{tab:ablation}
\end{table}

\begin{table}
\centering
\begin{tabular}{|l||c|}
\hline
    \textbf{Base / Loss} & \textbf{mIoU (\%)} \\ \hline 
    FCN / IID & 68.91 \\ \hline
    FCN / ASM & \textbf{72.14} \\ \hline \hline
    DeepLab / IID & 77.54 \\ \hline
    DeepLab / ASM & \textbf{78.05} \\ \hline \hline
    PSPNet / IID & 80.12 \\ \hline
    PSPNet / cGAN & 80.67 \\ \hline
    PSPNet / GAN & 80.74 \\ \hline
    PSPNet / ASM & \textbf{81.43} \\ \hline
\end{tabular}
\caption{Experimental results on PASVAL VOC with several base models, FCN~\cite{long2015fully}, DeepLab~\cite{chen2016deeplab}, and PSPNet~\cite{zhao2016pyramid}. The improvements by replacing pixel-wise losses with ASM are consistent across different base models. cGAN \cite{pix2pix2017} denotes the discriminator is conditioned on the input.}
\label{tab:summary}
\end{table}
 
We evaluate both figure-ground and semantic segmentation tasks via mean pixel-wise intersection-over-union (denoted as \textbf{mIoU})~\cite{long2015fully}. We first conduct ablation studies on both datasets to thoroughly analyze the effectiveness of using different layers of structure features in the structure analyzer. As summarized in Table~\ref{tab:ablation}, using low- to mid-level features (from $conv1$ to $conv2$) of structure analyzers yields the highest performance, $79.62\%$ and $71.60\%$ mIoU on Weizmann horse dataset and VOC dataset, respectively). We also report mIoU on VOC dataset using different base architectures as shown in Table~\ref{tab:summary}. Our proposed method achieves consistent improvements across all three base architectures, boosting mIoU by $3.23\%$ with FCN, $0.51\%$ with DeepLab and $1.31\%$ with PSPNet. ASM is also $0.71\%$ higher than GAN (incorporated with IID loss) on VOC dataset. We show some visual comparison in Fig.~\ref{fig:result}.

\vspace{6pt}

\noindent {\bf Boundary Localization Improvement.}
We argue that our proposed method is more sensitive to complex shapes of objects. We thus evaluate boundary localization using standard contour detection metrics~\cite{amfm_pami2011}. The contour detection metrics compute the correspondences between prediction boundaries and ground-truth boundaries, and summarize the results with precision, recall, and f-measure. We compare the results with different loss functions: IID, GAN and ASM on VOC validation set. As shown in Figure~\ref{fig:boundary}, ASM outperforms both IID and GAN among most categories and overall in boundary precision. (Detailed numbers including recall and f-measure can be found in the supplementary.) The boundaries of thin-structured objects, such as `bike' and `chair', are much better captured by ASM.

\begin{figure}
    \centering
    \includegraphics[width=1.0\linewidth]{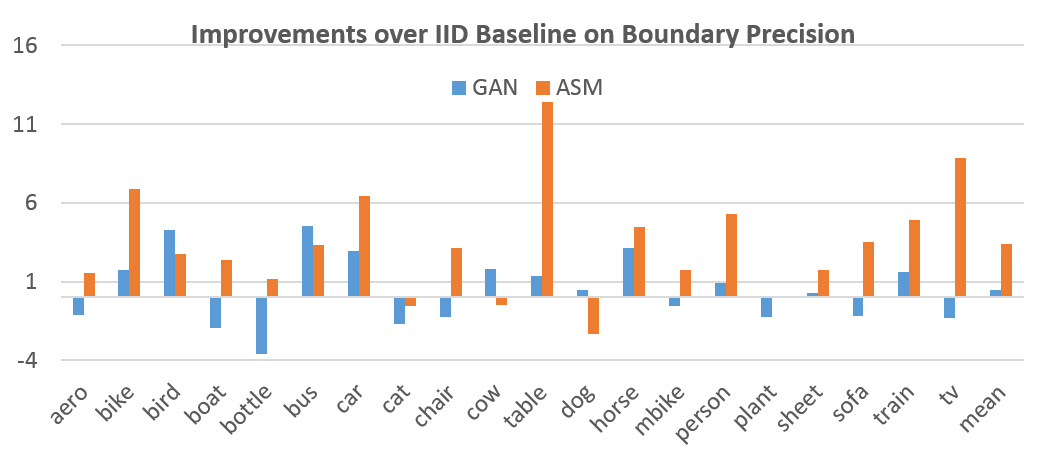}
    \caption{Improvements (\%) of per-class boundary precision of PSPNet~\cite{zhao2016pyramid} on PASCAL VOC 2012~\cite{everingham2010pascal} validation set. ASM outperforms IID and GAN in most categories.}
    \label{fig:boundary}
\end{figure}

\subsection{Main Results on Depth and Surface Normal}
\label{sec:exp_depth}

We compare ASM with pixel-wise L2 loss (IID) and cGAN~\cite{pix2pix2017} for monocular depth estimation. The performance is evaluated by the metrics proposed in ~\cite{eigen2014depth}, which are formulated as follows:


$\text{Abs Relative difference (rel)}: \frac{1}{|T|}\sum_{y \in T} |y-\hat{y}|/\hat{y}$

$\text{RMSE (rms)}: \sqrt{\frac{1}{|T|}\sum_{y \in T}\|y-\hat{y}\|^2}$

$\log_{10}: \sqrt{\frac{1}{|T|}\sum_{y \in T}\|\log y-\log\hat{y}\|^2}$

$\text{Threshold}: \% \text{ of } y_i \text{ s.t } \max(\frac{y_i}{\hat{y}_i}, \frac{\hat{y}_i}{y_i}) = \delta < thr$

As summarized in Table \ref{tab:depth}, our proposed method consistently improves among most metrics, including `rel', `rms', and accuracy with different thresholds. Note that `rel', `$\log_{10}$', and `rms' metrics reflect the {\it mean} value of error distributions, which penalizes more on larger errors, which are usually incurred by mislabeling of the ground truth. On the other hand, accuracy with thresholds evaluates the {\it median} of errors, which reflects more precisely the visual quality of predictions. We notice that our method outperforms others even more significantly in {\it  median} errors with smaller thresholds, which demonstrates the fine-grained structure discriminator power of ASM. The conclusion is consistent with our observations from visual comparisons in Fig.~\ref{fig:result}.

\begin{table}
  \centering
  \footnotesize
  \setlength\tabcolsep{3.0pt}
  \begin{tabular}{|l|c c c|c c c|}
    \hline
     & rel & $\log_{10}$ & rms & $\delta < 1.25$ & $\delta < 1.25^2$ & $\delta < 1.25^3$\\
    \hline \hline
    $\mathrm{IID}$ & 0.267 & 0.393 & 1.174 & 0.466 & 0.790 & 0.922 \\
    cGAN & 0.266 & 0.371 & 1.123 & 0.509 & 0.816 & 0.932 \\
    ASM & \textbf{0.252} & 0.405 & \textbf{1.079} & \textbf{0.540} & \textbf{0.834} & 0.929 \\
    \hline
  \end{tabular}
  \vspace{8pt}
  \caption{Depth estimation measurements on 2D-3D-S~\cite{2D-3D-S}. Note that lower is better for the first three columns, and higher is better for the last three columns. $\mathrm{IID}$ is L2.}
  \label{tab:depth}
\end{table}

We compare ASM with pixel-wise normalized L2 loss (IID) and cGAN~\cite{pix2pix2017} for surface normal prediction. To evaluate the prediction quality, we follow \cite{fouhey2013data} and report the results on several metrics, including the {\it mean} and {\it median} of error angles between the ground truth and prediction, and the percentage of error angles within $2.82^{\circ}$, $5.63^{\circ}$, $11.25^{\circ}$, $22.5^{\circ}$, $30^{\circ}$. The results are presented in Table \ref{tab:surface_normal}, and ASM improves metrics including {\it median} angles and percentage of angles within $2.82^{\circ}$, $5.63^{\circ}$, $11.25^{\circ}$ by large margin. Similar to accuracy with thresholds for depth estimation, metrics with smaller angles are more consistent with visual quality. We conclude that ASM captures most details and outperforms baselines in most cases. We observe in Fig.~\ref{fig:result} prominent visual improvements on thin structures, such as bookcase shelves. The surface normal prediction of larger objects, including wall and ceiling, are more uniform. Also, the contrast between adjacent surfaces are sharper with ASM.

\begin{table}
  \centering
  \footnotesize
  \setlength\tabcolsep{3.0pt}
  \begin{tabular}{|l|c c|c c c c c|}
    \hline
     & \multicolumn{2}{c|}{Angle Distance} & \multicolumn{5}{c|}{Within $t^{\circ}$ Deg.}\\
     & Mean & Median & $2.82^{\circ}$ & $5.63^{\circ}$ & $11.25^{\circ}$ & $22.5^{\circ}$ & $30^{\circ}$\\
    \hline \hline
    $\mathrm{IID}$ & 16.84 & 9.20 & 12.90 & 31.12 & 58.01 & 77.37 & 82.76\\
    cGAN & 17.12 & 9.04 & 12.28 & 31.52 & 58.52 & 77.00 & 82.32\\
    ASM & 16.98 & \textbf{8.28} & \textbf{17.58} & \textbf{35.48} & \textbf{63.29} & \textbf{78.50} & 82.47\\
    \hline
  \end{tabular}
  \vspace{8pt}
  \caption{Surface normal estimation measurements on 2D-3D-S \cite{2D-3D-S}. Note that lower is better for the first two columns, and higher is better for the last three columns. $\mathrm{IID}$ is normalized L2.}
  \label{tab:surface_normal}
\end{table}

\vspace{6pt}
\noindent {\bf Instance- and Semantic-level Analysis.}
Just as mIoU metric in semantic segmentation, the metrics proposed by \cite{eigen2014depth} for depth and by \cite{fouhey2013data} for surface normal are biased toward larger objects. We thus follow \cite{ke2018adaptive} to evaluate instance-wise metrics to attenuate the bias and fairly evaluate the performance on smaller objects. We collect semantic instance masks on 2D-3D-S dataset and formulate the instance-wise metric as $\frac{\sum_{i \in I_c} M_{i,c}}{|I_c|}$, where $I_c$ denotes the set of instances in class $c$ and $M_{i,c}$ is the metric for depth or surface normal of instance $i$ in class $c$. (We do not use instance- or semantic-level information during training.) As shown in Fig~\ref{fig:inst_normal}, we demonstrate that ASM improves instance-wise percentage of angles within $11.25^{\circ}$ consistently among all categories. The instance-wise analysis for depth estimation can be found in the supplementary. 

\begin{figure}
    \centering
    \includegraphics[width=1.0\linewidth]{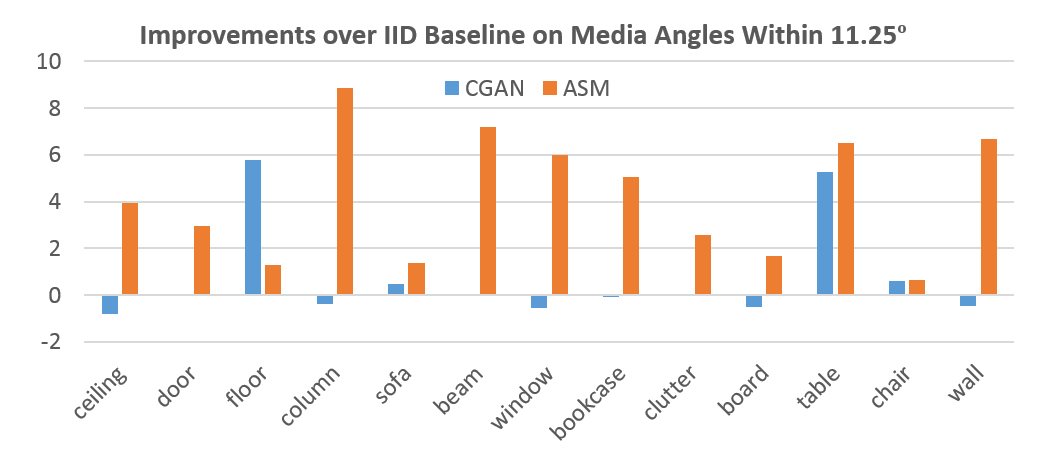}
    \caption{Improvements (\%) of instance-average surface normal metric on 2D-3D-S \cite{2D-3D-S}. The categories with most improvements are `column', `beam', `table', and `wall'.}
    \label{fig:inst_normal}
\end{figure}

\begin{figure}
    \centering
    \includegraphics[width=1.0\linewidth]{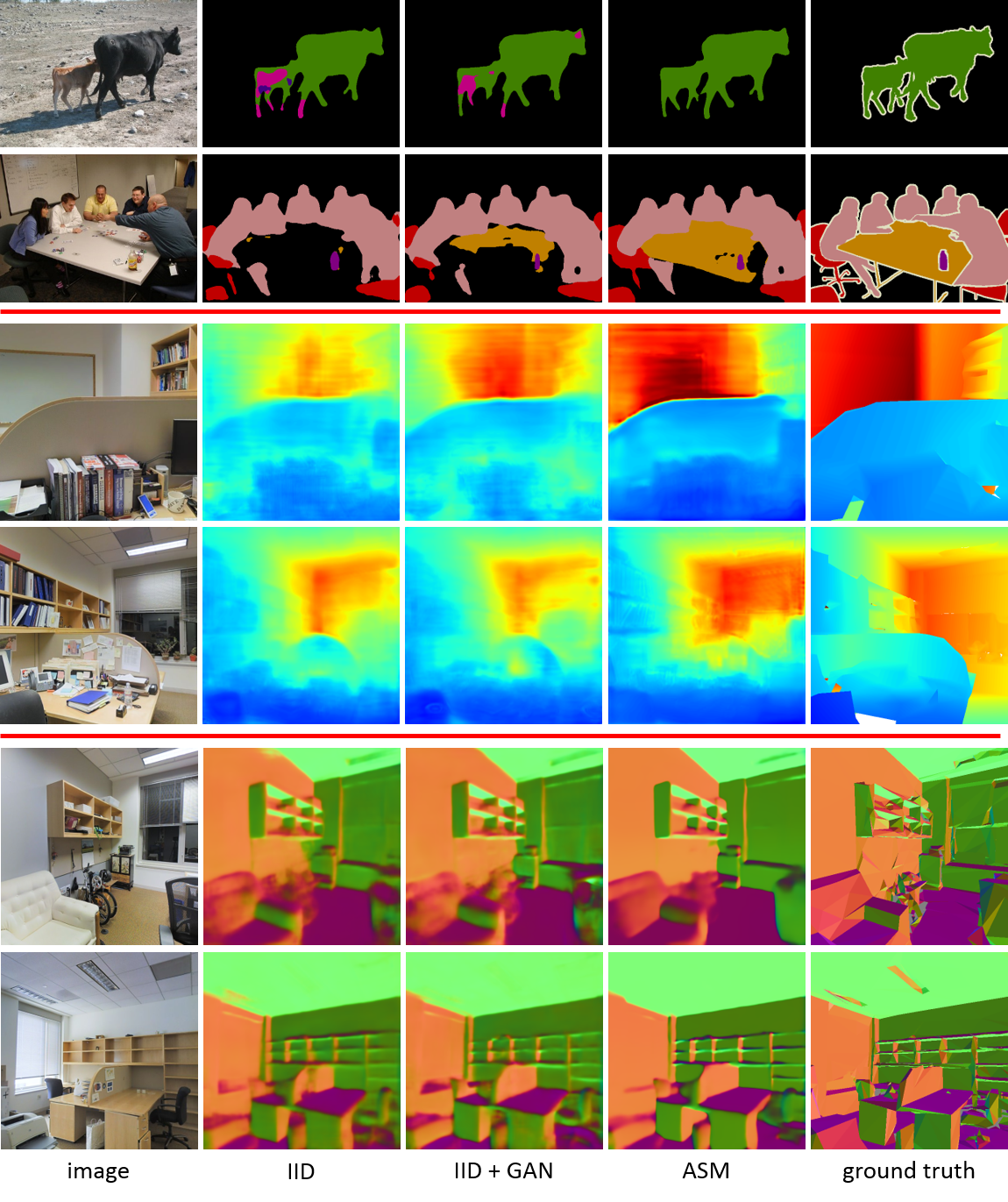}
    \caption{Visual quality comparison for semantic segmentation (top two rows) on VOC  \cite{everingham2010pascal} validation set, monocular depth estimation (middle two rows) and surface normal prediction (bottom two rows) on \cite{2D-3D-S}. Left to right: Images, predictions from training with IID, IID+(c)GAN, ASM, and ground truth masks.}
    \label{fig:result}
\end{figure}


\section{Conclusion}
\label{sec:con}

We proposed a novel framework, Adversarial Structure Matching, for training structured prediction networks. We provided theoretical analyses and extensive experiments to demonstrate the efficacy of ASM. We concluded that multi-scale hard negative structural errors provide better supervision than the conventional pixel-wise IID losses (or incorporated with structure priors) in different structured prediction tasks, namely, semantic segmentation, monocular depth estimation, and surface normal prediction.\\

\noindent {\bf Acknowledgements.} This research was supported, in part, by Berkeley Deep Drive, NSF (IIS-1651389), DARPA, and US Government fund through Etegent Technologies on Low-Shot Detection in Remote Sensing Imagery.  The views, opinions and/or findings expressed should not be interpreted as representing the official views or policies of NSF, DARPA, or the U.S. Government.





\clearpage
{\small
\bibliographystyle{ieee}
\bibliography{egbib}

\begin{thebibliography}{10}\itemsep=-1pt

\bibitem{amfm_pami2011}
P.~Arbelaez, M.~Maire, C.~Fowlkes, and J.~Malik.
\newblock Contour detection and hierarchical image segmentation.
\newblock {\em TPAMI}, 2011.

\bibitem{2D-3D-S}
I.~Armeni, A.~Sax, A.~R. Zamir, and S.~Savarese.
\newblock Joint 2d-3d-semantic data for indoor scene understanding.
\newblock {\em arXiv preprint arXiv:1702.01105}, 2017.

\bibitem{bertasius2016semantic}
G.~Bertasius, J.~Shi, and L.~Torresani.
\newblock Semantic segmentation with boundary neural fields.
\newblock In {\em CVPR}, 2016.

\bibitem{bertasius2016convolutional}
G.~Bertasius, L.~Torresani, S.~X. Yu, and J.~Shi.
\newblock Convolutional random walk networks for semantic image segmentation.
\newblock In {\em CVPR}, 2017.

\bibitem{borenstein2002horse}
E.~Borenstein and S.~Ullman.
\newblock Class-specific, top-down segmentation.
\newblock In {\em ECCV}, 2002.

\bibitem{chen2016semantic}
L.-C. Chen, J.~T. Barron, G.~Papandreou, K.~Murphy, and A.~L. Yuille.
\newblock Semantic image segmentation with task-specific edge detection using
  cnns and a discriminatively trained domain transform.
\newblock In {\em CVPR}, 2016.

\bibitem{chen2016deeplab}
L.-C. Chen, G.~Papandreou, I.~Kokkinos, K.~Murphy, and A.~L. Yuille.
\newblock Deeplab: Semantic image segmentation with deep convolutional nets,
  atrous convolution, and fully connected crfs.
\newblock {\em arXiv preprint arXiv:1606.00915}, 2016.

\bibitem{chen2017rethinking}
L.-C. Chen, G.~Papandreou, F.~Schroff, and H.~Adam.
\newblock Rethinking atrous convolution for semantic image segmentation.
\newblock {\em arXiv preprint arXiv:1706.05587}, 2017.

\bibitem{chen2015learning}
L.-C. Chen, A.~Schwing, A.~Yuille, and R.~Urtasun.
\newblock Learning deep structured models.
\newblock In {\em ICML}, 2015.

\bibitem{cheng2018depth}
X.~Cheng, P.~Wang, and R.~Yang.
\newblock Depth estimation via affinity learned with convolutional spatial
  propagation network.
\newblock In {\em ECCV}, 2018.

\bibitem{couprie2013indoor}
C.~Couprie, C.~Farabet, L.~Najman, and Y.~LeCun.
\newblock Indoor semantic segmentation using depth information.
\newblock {\em arXiv preprint arXiv:1301.3572}, 2013.

\bibitem{eigen2015predicting}
D.~Eigen and R.~Fergus.
\newblock Predicting depth, surface normals and semantic labels with a common
  multi-scale convolutional architecture.
\newblock In {\em ICCV}, 2015.

\bibitem{eigen2014depth}
D.~Eigen, C.~Puhrsch, and R.~Fergus.
\newblock Depth map prediction from a single image using a multi-scale deep
  network.
\newblock In {\em NIPS}, 2014.

\bibitem{everingham2010pascal}
M.~Everingham, L.~Van~Gool, C.~K. Williams, J.~Winn, and A.~Zisserman.
\newblock The pascal visual object classes (voc) challenge.
\newblock {\em IJCV}, 2010.

\bibitem{fouhey2013data}
D.~F. Fouhey, A.~Gupta, and M.~Hebert.
\newblock Data-driven 3d primitives for single image understanding.
\newblock In {\em ICCV}, 2013.

\bibitem{gatys2015neural}
L.~A. Gatys, A.~S. Ecker, and M.~Bethge.
\newblock A neural algorithm of artistic style.
\newblock {\em arXiv preprint arXiv:1508.06576}, 2015.

\bibitem{goodfellow2014generative}
I.~Goodfellow, J.~Pouget-Abadie, M.~Mirza, B.~Xu, D.~Warde-Farley, S.~Ozair,
  A.~Courville, and Y.~Bengio.
\newblock Generative adversarial nets.
\newblock In {\em NIPS}, 2014.

\bibitem{gupta2014learning}
S.~Gupta, R.~Girshick, P.~Arbel{\'a}ez, and J.~Malik.
\newblock Learning rich features from rgb-d images for object detection and
  segmentation.
\newblock In {\em ECCV}, 2014.

\bibitem{gygli2017value}
M.~Gygli, M.~Norouzi, and A.~Angelova.
\newblock Deep value networks learn to evaluate and iteratively refine
  structured outputs.
\newblock In {\em ICML}, 2017.

\bibitem{he2016deep}
K.~He, X.~Zhang, S.~Ren, and J.~Sun.
\newblock Deep residual learning for image recognition.
\newblock In {\em CVPR}, 2016.

\bibitem{hwang2015pixel}
J.-J. Hwang and T.-L. Liu.
\newblock Pixel-wise deep learning for contour detection.
\newblock In {\em ICLR Workshop}, 2015.

\bibitem{pix2pix2017}
P.~Isola, J.-Y. Zhu, T.~Zhou, and A.~A. Efros.
\newblock Image-to-image translation with conditional adversarial networks.
\newblock {\em CVPR}, 2017.

\bibitem{karsch2014depth}
K.~Karsch, C.~Liu, and S.~B. Kang.
\newblock Depth transfer: Depth extraction from video using non-parametric
  sampling.
\newblock {\em TPAMI}, 2014.

\bibitem{ke2018adaptive}
T.-W. Ke, J.-J. Hwang, Z.~Liu, and S.~X. Yu.
\newblock Adaptive affinity field for semantic segmentation.
\newblock 2018.

\bibitem{kendall2017uncertainties}
A.~Kendall and Y.~Gal.
\newblock What uncertainties do we need in bayesian deep learning for computer
  vision?
\newblock In {\em NIPS}, 2017.

\bibitem{kim2016unified}
S.~Kim, K.~Park, K.~Sohn, and S.~Lin.
\newblock Unified depth prediction and intrinsic image decomposition from a
  single image via joint convolutional neural fields.
\newblock In {\em ECCV}, 2016.

\bibitem{kokkinos2017ubernet}
I.~Kokkinos.
\newblock Ubernet: Training a universal convolutional neural network for low-,
  mid-, and high-level vision using diverse datasets and limited memory.
\newblock In {\em CVPR}, 2017.

\bibitem{krahenbuhl2011efficient}
P.~Kr{\"a}henb{\"u}hl and V.~Koltun.
\newblock Efficient inference in fully connected crfs with gaussian edge
  potentials.
\newblock In {\em NIPS}, 2011.

\bibitem{krizhevsky2012imagenet}
A.~Krizhevsky, I.~Sutskever, and G.~E. Hinton.
\newblock Imagenet classification with deep convolutional neural networks.
\newblock In {\em NIPS}, 2012.

\bibitem{ladicky2014pulling}
L.~Ladicky, J.~Shi, and M.~Pollefeys.
\newblock Pulling things out of perspective.
\newblock In {\em CVPR}, 2014.

\bibitem{laina2016deeper}
I.~Laina, C.~Rupprecht, V.~Belagiannis, F.~Tombari, and N.~Navab.
\newblock Deeper depth prediction with fully convolutional residual networks.
\newblock In {\em 3DV}, 2016.

\bibitem{li2015depth}
B.~Li, C.~Shen, Y.~Dai, A.~Van Den~Hengel, and M.~He.
\newblock Depth and surface normal estimation from monocular images using
  regression on deep features and hierarchical crfs.
\newblock In {\em CVPR}, 2015.

\bibitem{li2017two}
J.~Li, R.~Klein, and A.~Yao.
\newblock A two-streamed network for estimating fine-scaled depth maps from
  single rgb images.
\newblock In {\em ICCV}, 2017.

\bibitem{li2017not}
X.~Li, Z.~Liu, P.~Luo, C.~C. Loy, and X.~Tang.
\newblock Not all pixels are equal: Difficulty-aware semantic segmentation via
  deep layer cascade.
\newblock In {\em CVPR}, 2017.

\bibitem{lin2016efficient}
G.~Lin, C.~Shen, A.~van~den Hengel, and I.~Reid.
\newblock Efficient piecewise training of deep structured models for semantic
  segmentation.
\newblock In {\em CVPR}, 2016.

\bibitem{liu2015deep}
F.~Liu, C.~Shen, and G.~Lin.
\newblock Deep convolutional neural fields for depth estimation from a single
  image.
\newblock In {\em CVPR}, 2015.

\bibitem{liu2016learning}
F.~Liu, C.~Shen, G.~Lin, and I.~D. Reid.
\newblock Learning depth from single monocular images using deep convolutional
  neural fields.
\newblock {\em TPAMI}, 2016.

\bibitem{liu2017learning}
S.~Liu, S.~De~Mello, J.~Gu, G.~Zhong, M.-H. Yang, and J.~Kautz.
\newblock Learning affinity via spatial propagation networks.
\newblock In {\em NIPS}, 2017.

\bibitem{liu2015semantic}
Z.~Liu, X.~Li, P.~Luo, C.-C. Loy, and X.~Tang.
\newblock Semantic image segmentation via deep parsing network.
\newblock In {\em CVPR}, 2015.

\bibitem{long2015fully}
J.~Long, E.~Shelhamer, and T.~Darrell.
\newblock Fully convolutional networks for semantic segmentation.
\newblock In {\em CVPR}, 2015.

\bibitem{luc2016semantic}
P.~Luc, C.~Couprie, S.~Chintala, and J.~Verbeek.
\newblock Semantic segmentation using adversarial networks.
\newblock {\em NIPS Workshop}, 2016.

\bibitem{maire2016affinity}
M.~Maire, T.~Narihira, and S.~X. Yu.
\newblock Affinity cnn: Learning pixel-centric pairwise relations for
  figure/ground embedding.
\newblock In {\em CVPR}, 2016.

\bibitem{misra2016cross}
I.~Misra, A.~Shrivastava, A.~Gupta, and M.~Hebert.
\newblock Cross-stitch networks for multi-task learning.
\newblock In {\em CVPR}, 2016.

\bibitem{MMS:CVPR:2018}
M.~Mostajabi, M.~Maire, and G.~Shakhnarovich.
\newblock Regularizing deep networks by modeling and predicting label
  structure.
\newblock In {\em CVPR}, 2018.

\bibitem{noh2015learning}
H.~Noh, S.~Hong, and B.~Han.
\newblock Learning deconvolution network for semantic segmentation.
\newblock In {\em CVPR}, 2015.

\bibitem{ren2012rgb}
X.~Ren, L.~Bo, and D.~Fox.
\newblock Rgb-(d) scene labeling: Features and algorithms.
\newblock In {\em CVPR}, 2012.

\bibitem{richter2017playing}
S.~R. Richter, Z.~Hayder, and V.~Koltun.
\newblock Playing for benchmarks.
\newblock In {\em ICCV}, volume~2, 2017.

\bibitem{ronneberger2015u}
O.~Ronneberger, P.~Fischer, and T.~Brox.
\newblock U-net: Convolutional networks for biomedical image segmentation.
\newblock In {\em MICCAI}, 2015.

\bibitem{roy2016monocular}
A.~Roy and S.~Todorovic.
\newblock Monocular depth estimation using neural regression forest.
\newblock In {\em CVPR}, 2016.

\bibitem{shi2000normalized}
J.~Shi and J.~Malik.
\newblock Normalized cuts and image segmentation.
\newblock {\em TPAMI}, 2000.

\bibitem{silberman2012indoor}
N.~Silberman, D.~Hoiem, P.~Kohli, and R.~Fergus.
\newblock Indoor segmentation and support inference from rgbd images.
\newblock In {\em ECCV}, 2012.

\bibitem{wang2015towards}
P.~Wang, X.~Shen, Z.~Lin, S.~Cohen, B.~Price, and A.~L. Yuille.
\newblock Towards unified depth and semantic prediction from a single image.
\newblock In {\em CVPR}, 2015.

\bibitem{wang2016surge}
P.~Wang, X.~Shen, B.~Russell, S.~Cohen, B.~Price, and A.~L. Yuille.
\newblock Surge: Surface regularized geometry estimation from a single image.
\newblock In {\em NIPS}, 2016.

\bibitem{xie2016top}
S.~Xie, X.~Huang, and Z.~Tu.
\newblock Top-down learning for structured labeling with convolutional
  pseudoprior.
\newblock In {\em ECCV}, 2016.

\bibitem{xu2018pad}
D.~Xu, W.~Ouyang, X.~Wang, and N.~Sebe.
\newblock Pad-net: Multi-tasks guided prediction-and-distillation network for
  simultaneous depth estimation and scene parsing.
\newblock {\em arXiv preprint arXiv:1805.04409}, 2018.

\bibitem{xu2017multi}
D.~Xu, E.~Ricci, W.~Ouyang, X.~Wang, and N.~Sebe.
\newblock Multi-scale continuous crfs as sequential deep networks for monocular
  depth estimation.
\newblock In {\em CVPR}, volume~1, 2017.

\bibitem{yu2015multi}
F.~Yu and V.~Koltun.
\newblock Multi-scale context aggregation by dilated convolutions.
\newblock In {\em ICLR}, 2016.

\bibitem{zamir2018taskonomy}
A.~R. Zamir, A.~Sax, W.~Shen, L.~Guibas, J.~Malik, and S.~Savarese.
\newblock Taskonomy: Disentangling task transfer learning.
\newblock In {\em CVPR}, 2018.

\bibitem{zhang2018unreasonable}
R.~Zhang, P.~Isola, A.~A. Efros, E.~Shechtman, and O.~Wang.
\newblock The unreasonable effectiveness of deep features as a perceptual
  metric.
\newblock In {\em CVPR}, 2018.

\bibitem{zhang2018joint}
Z.~Zhang, Z.~Cui, C.~Xu, Z.~Jie, X.~Li, and J.~Yang.
\newblock Joint task-recursive learning for semantic segmentation and depth
  estimation.
\newblock In {\em ECCV}, 2018.

\bibitem{zhao2016pyramid}
H.~Zhao, J.~Shi, X.~Qi, X.~Wang, and J.~Jia.
\newblock Pyramid scene parsing network.
\newblock In {\em CVPR}, 2017.

\bibitem{zheng2015conditional}
S.~Zheng, S.~Jayasumana, B.~Romera-Paredes, V.~Vineet, Z.~Su, D.~Du, C.~Huang,
  and P.~H. Torr.
\newblock Conditional random fields as recurrent neural networks.
\newblock In {\em ICCV}, 2015.

\bibitem{zhou2018interpreting}
B.~Zhou, D.~Bau, A.~Oliva, and A.~Torralba.
\newblock Interpreting deep visual representations via network dissection.
\newblock {\em PAMI}, 2018.

\bibitem{zhuo2015indoor}
W.~Zhuo, M.~Salzmann, X.~He, and M.~Liu.
\newblock Indoor scene structure analysis for single image depth estimation.
\newblock In {\em CVPR}, 2015.

\end{thebibliography}
}

\clearpage
\section{Supplementary}
\label{sec:app}

\begin{figure}
    \centering
    \includegraphics[width=1.0\linewidth]{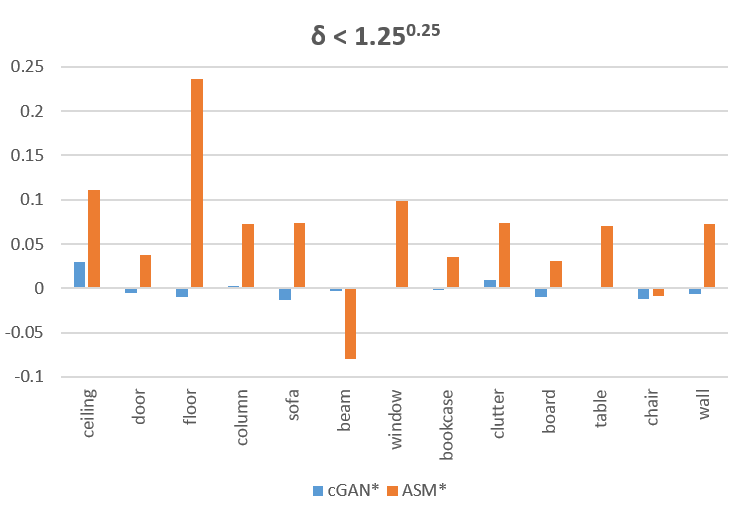}
    \caption{Improvements of cGAN and ASM over IID baseline with the instance-average depth metric on 2D-3D-Semantics~\cite{2D-3D-S} dataset. All networks are jointly trained with depth and surface normal estimation.}
    \label{fig:inst_depth_joint}
\end{figure}

\section{Joint Training of Depth and Surface Normal}

Aside from predicting depth and surface normal as two independent tasks, we also carry out experiments to jointly train both tasks. For the structured prediction network, we use the same architecture as the one for single task training. The only difference is that there are two output layers--one for depth estimation and the other for surface normal prediction. For the structure analyzer, we double the number of channels per hidden layer, and concatenate `depth' and `surface normal' after $conv1$. The structure analyzer also has two output layers for depth and surface normal prediction, respectively. The architecture detail is described in Section \ref{subsec:network_architecture}.

We conclude that ASM provides better supervisory signals, which take into consideration the joint structures of depth and surface normal, than $\mathrm{IID}$ losses, $L_2$ and normalized $L_2$ for depth and surface normal. As summarized in Table \ref{tab:app_depth}, our proposed method consistently improves all the metrics by a large margin. It means that the surface normal information can be used to guide depth estimation in ASM. We present some visual comparison in Figure \ref{fig:stanford_joint}. We notice prominent improvements on the object boundaries. Also, thin structures are more conspicuous.

To show the competitive performance of our baseline methods, we also train the high performing model FCRN~\cite{laina2016deeper} with ResNet-50~\cite{he2016deep} backbone pre-trained on ImageNet~\cite{krizhevsky2012imagenet} using $\mathrm{IID}$ losses.  Note that we do not pre-train our UNet~\cite{ronneberger2015u} on ImageNet~\cite{krizhevsky2012imagenet}. The results are summarized in Table~\ref{tab:app_depth}. We conclude that our UNet baseline is competitive with other architectures.

We also evaluate baselines and our proposed method by instance-wise metrics. We plot the improvement of each categories, and the results are presented in Figure~\ref{fig:inst_depth_joint}. Without training using instance- or semantic-level information, ASM outperforms IID baseline and cGAN among most categories.

\begin{table}
  \centering
  \footnotesize
  \setlength\tabcolsep{2.0pt}
  \begin{tabular}{|l|c c|c c c|}
    \hline
     & \multicolumn{2}{c|}{FCRN~\cite{laina2016deeper}} & \multicolumn{3}{c|}{UNet}\\
     & $\mathrm{IID}$ & $\mathrm{IID}$* & $\mathrm{IID}$* & cGAN* & ASM*\\
     \hline\hline
     rel & 0.228 & 0.225 & 0.241 & 0.247 & \textbf{0.215}\\
     $\log_{10}$ & 0.298 & 0.297 & 0.341 & 0.362 & \textbf{0.336}\\
     rms & 0.881 & 0.874 & 1.046 & 1.095 & \textbf{0.936}\\
     \hline
     $\delta < 1.25^{0.25}$ & 0.193 & 0.189 & 0.132 & 0.128 & \textbf{0.204}\\
     $\delta < 1.25^{0.5}$ & 0.369 & 0.364 & 0.267 & 0.259 & \textbf{0.385}\\
     $\delta < 1.25$ & 0.640 & 0.634 & 0.530 & 0.509 & \textbf{0.653}\\
     $\delta < 1.25^2$ & 0.890 & 0.891 & 0.854 & 0.831 & \textbf{0.888}\\
     $\delta < 1.25^3$ & 0.965 & 0.966 & 0.951 & 0.939 & \textbf{0.958}\\
     \hline
  \end{tabular}
  \vspace{8pt}
  \caption{Depth estimation measurements on 2D-3D-Semantics~\cite{2D-3D-S} dataset. Note that lower is better for the first three rows, and higher is better for the last five rows. * denotes joint training of depth and surface normal prediction.}
  \label{tab:app_depth}
\end{table}

\begin{figure*}
    \centering
    \includegraphics[width=0.9\linewidth]{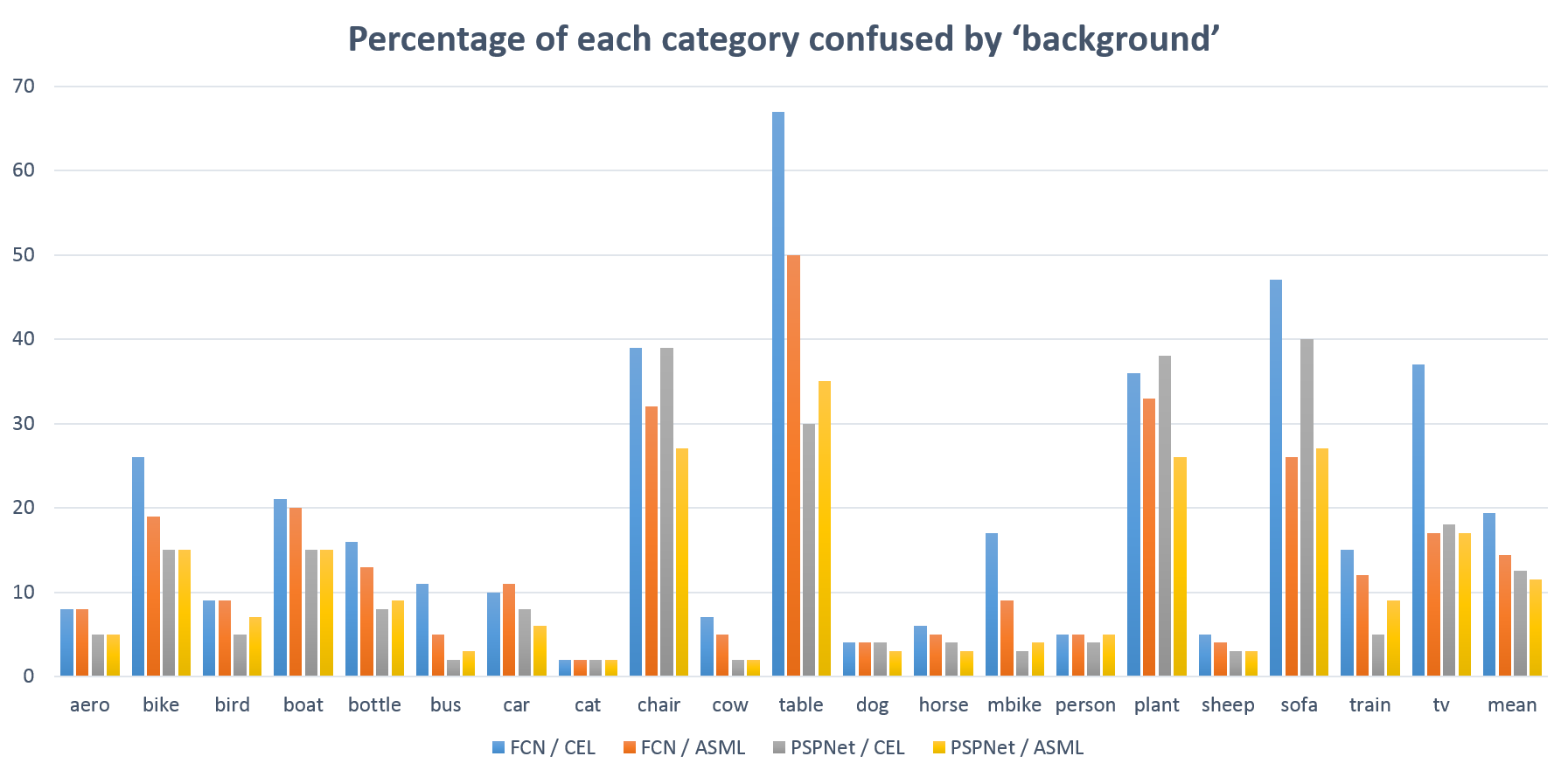}
    \caption{Percentage of each category confused by `background' (the lower the better). The categories with much confusion reduction are `chair', `plant',`sofa', and `tv'.}
    \label{fig:confusion}
\end{figure*}

\section{Confusing Context Improvement for Segmentation}

We next demonstrate the robustness of our proposed method under confusing context. We first calculate the confusion matrix of pixel accuracy on PASCAL VOC 2012~\cite{everingham2010pascal} validation set. We identify that `background' is biggest confuser for most of the categories and hence we summarize the percentage of confusion in Fig.~\ref{fig:confusion} (i.e., the `background' column from the confusion matrix). ASM reduces the overall confusion caused by `background' from $19.4\%$ to $14.45\%$ on FCN and from $12.5\%$ to $11.5\%$ on PSPNet with $8.7\%$ relative error reduction. Large improvements come from resolving confusion of `chair', `plant', `sofa', and `tv'.

\section{Network Architecture}
\label{subsec:network_architecture}
We use UNet~\cite{ronneberger2015u} as our structure analyzer for all experiments and structured prediction network for depth estimation and surface normal prediction on 2D-3D-Semantics~\cite{2D-3D-S} dataset. The detail architecture of structure analyzers is summarized in Table~\ref{tab:analyzer_weizmann}, \ref{tab:analyzer_pascal}, \ref{tab:analyzer_2d3ds_single}, and  \ref{tab:analyzer_2d3ds_joint}. The architecture for structured prediction network on 2D-3D-S~\cite{2D-3D-S} is shown in Table~\ref{tab:structure_predictor_unet}. We set `stride' to $2$ for each $conv$ layers in the encoder of UNet. We set `stride' to $1$ and apply $bilinear$ upsampling in the decoder of UNet.

\begin{table*}
\parbox{.48\linewidth}{
  \centering
    \begin{tabular}{|l|c|c|}
      \hline
      layer name & connection & dimension\\
      \hline\hline
      $conv1$ & $input$ & $5 \times 5, 32$\\
      $conv2$ & $conv1$ & $5 \times 5, 64$\\
      $conv3$ & $conv2$ & $5 \times 5, 128$\\
      $conv4$ & $conv3$ & $5 \times 5, 256$\\
      $conv5$ & $conv4\,\&\,conv3$ & $3 \times 3, 128$\\
      $conv6$ & $conv5\,\&\,conv2$ & $3 \times 3, 64$\\
      $conv7$ & $conv6\,\&\,conv1$ & $3 \times 3, 32$\\
      $output$ & $conv7$ & $1 \times 1, 1$\\
      \hline
    \end{tabular}
    \vspace{8pt}
    \caption{Architecture of structure analyzer for figure-ground segmentation on the Weizmann horse~\cite{borenstein2002horse}.}
    \label{tab:analyzer_weizmann}
    } \hfill
    \parbox{.48\linewidth}{
    \centering
    \begin{tabular}{|l|c|c|}
      \hline
      layer name & connection & dimension\\
      \hline\hline
      $conv1$ & $input$ & $3 \times 3, 128$\\
      $conv2$ & $conv1$ & $3 \times 3, 256$\\
      $conv3$ & $conv2$ & $3 \times 3, 256$\\
      $conv4$ & $conv3\,\&\,conv2$ & $3 \times 3, 256$\\
      $conv5$ & $conv4\,\&\,conv1$ & $3 \times 3, 128$\\
      $output$ & $conv5$ & $1 \times 1, 21$\\
      \hline
    \end{tabular}
    \vspace{8pt}
    \caption{Architecture of structure analyzer for semantic segmentation on PASCAL VOC 2012~\cite{everingham2010pascal}.}
    \label{tab:analyzer_pascal}
    }
\end{table*}

\begin{table*}
\parbox{.48\linewidth}{
  \centering
    \begin{tabular}{|l|c|c|}
      \hline
      layer name & connection & dimension\\
      \hline\hline
      $conv1$ & $input$ & $3 \times 3, 32$\\
      $conv2$ & $conv1$ & $3 \times 3, 64$\\
      $conv3$ & $conv2$ & $3 \times 3, 128$\\
      $conv4$ & $conv3$ & $3 \times 3, 128$\\
      $conv5$ & $conv4$ & $3 \times 3, 128$\\
      $conv6$ & $conv5\,\&\,conv4$ & $3 \times 3, 128$\\
      $conv7$ & $conv6\,\&\,conv3$ & $3 \times 3, 128$\\
      $conv8$ & $conv7\,\&\,conv2$ & $3 \times 3, 64$\\
      $conv9$ & $conv8\,\&\,conv1$ & $3 \times 3, 32$\\
      $output$ & $conv9$ & $1 \times 1, 1 (or 3)$\\
      \hline
    \end{tabular}
    \vspace{8pt}
    \caption{Architecture of structure analyzer for depth or surface normal estimation on 2D-3D-S~\cite{2D-3D-S}. Note that the output channel is 1 for depth and 3 for surface normal.}
    \label{tab:analyzer_2d3ds_single}
    } \hfill
    \parbox{.48\linewidth}{
    \centering
    \begin{tabular}{|l|c|c|}
      \hline
      layer name & connection & dimension\\
      \hline\hline
      $conv1\_1$ & $input\_depth$ & $3 \times 3, 32$\\
      $conv1\_2$ & $input\_normal$ & $3 \times 3, 32$\\
      $conv2$ & $conv1^*$ & $3 \times 3, 128$\\
      $conv3$ & $conv2$ & $3 \times 3, 256$\\
      $conv4$ & $conv3$ & $3 \times 3, 256$\\
      $conv5$ & $conv4$ & $3 \times 3, 256$\\
      $conv6$ & $conv5\,\&\,conv4$ & $3 \times 3, 256$\\
      $conv7$ & $conv6\,\&\,conv3$ & $3 \times 3, 256$\\
      $conv8$ & $conv7\,\&\,conv2$ & $3 \times 3, 128$\\
      $conv9$ & $conv8\,\&\,conv1^*$ & $3 \times 3, 64$\\
      $output\_depth$ & $conv9$ & $1 \times 1 , 1$\\
      $output\_normal$ & $conv9$ & $1 \times 1, 3$\\
      \hline
    \end{tabular}
    \vspace{8pt}
    \caption{Architecture of structure analyzer for joint training of depth and surface normal estimation on 2D-3D-S~\cite{2D-3D-S}. $conv1^*$ is the concatenation of $conv1\_1$ and $conv1\_2$.}
    \label{tab:analyzer_2d3ds_joint}
    }
\end{table*}

\begin{table}
  \centering
    \begin{tabular}{|l|c|c|}
      \hline
      layer name & connection & dimension\\
      \hline\hline
      $conv1$ & $input$ & $3 \times 3, 64$\\
      $conv2$ & $conv1$ & $3 \times 3, 128$\\
      $conv3$ & $conv2$ & $3 \times 3, 256$\\
      $conv4$ & $conv3$ & $3 \times 3, 512$\\
      $conv5$ & $conv4$ & $3 \times 3, 512$\\
      $conv6$ & $conv5$ & $3 \times 3, 512$\\
      $conv7$ & $conv6$ & $3 \times 3, 512$\\
      $conv8$ & $conv7$ & $3 \times 3, 512$\\
      $conv9$ & $conv8$ & $3 \times 3, 512$\\
      $conv10$ & $conv9\,\&\,conv8$ & $[3 \times 3, 512] \times 3$\\
      $conv11$ & $conv10\,\&\,conv7$ & $[3 \times 3, 512] \times 3$\\
      $conv12$ & $conv11\,\&\,conv6$ & $[3 \times 3, 512] \times 3$\\
      $conv13$ & $conv12\,\&\,conv5$ & $[3 \times 3, 512] \times 3$\\
      $conv14$ & $conv13\,\&\,conv4$ & $[3 \times 3, 512] \times 3$\\
      $conv15$ & $conv14\,\&\,conv3$ & $[3 \times 3, 256] \times 3$\\
      $conv16$ & $conv15\,\&\,conv2$ & $[3 \times 3, 128] \times 3$\\
      $conv17$ & $conv16\,\&\,conv1$ & $[3 \times 3, 64] \times 3$\\
      $output$ & $conv5$ & $1 \times 1, 1 (or 3)$\\
      \hline
    \end{tabular}
    \caption{Architecture of structured prediction network for depth estimation and surface normal prediction on 2D-3D-Semantics~\cite{2D-3D-S}.}
    \label{tab:structure_predictor_unet}
\end{table}

\begin{figure}
    \centering
    \includegraphics[width=1.0\linewidth]{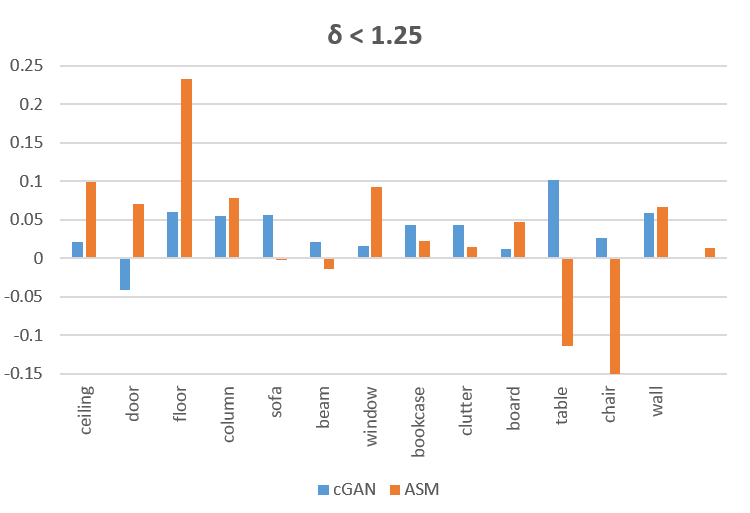}
    \caption{Improvement of instance-average depth metric on 2D-3D-Semantics~\cite{2D-3D-S}. The categrories with largest improvement is `floor', `ceiling' and `window'. Note that the results are based on single task training of depth estimation.}
    \label{fig:inst_depth}
\end{figure}

\begin{figure*}[hb]
    \centering
    \includegraphics[width=0.97\linewidth]{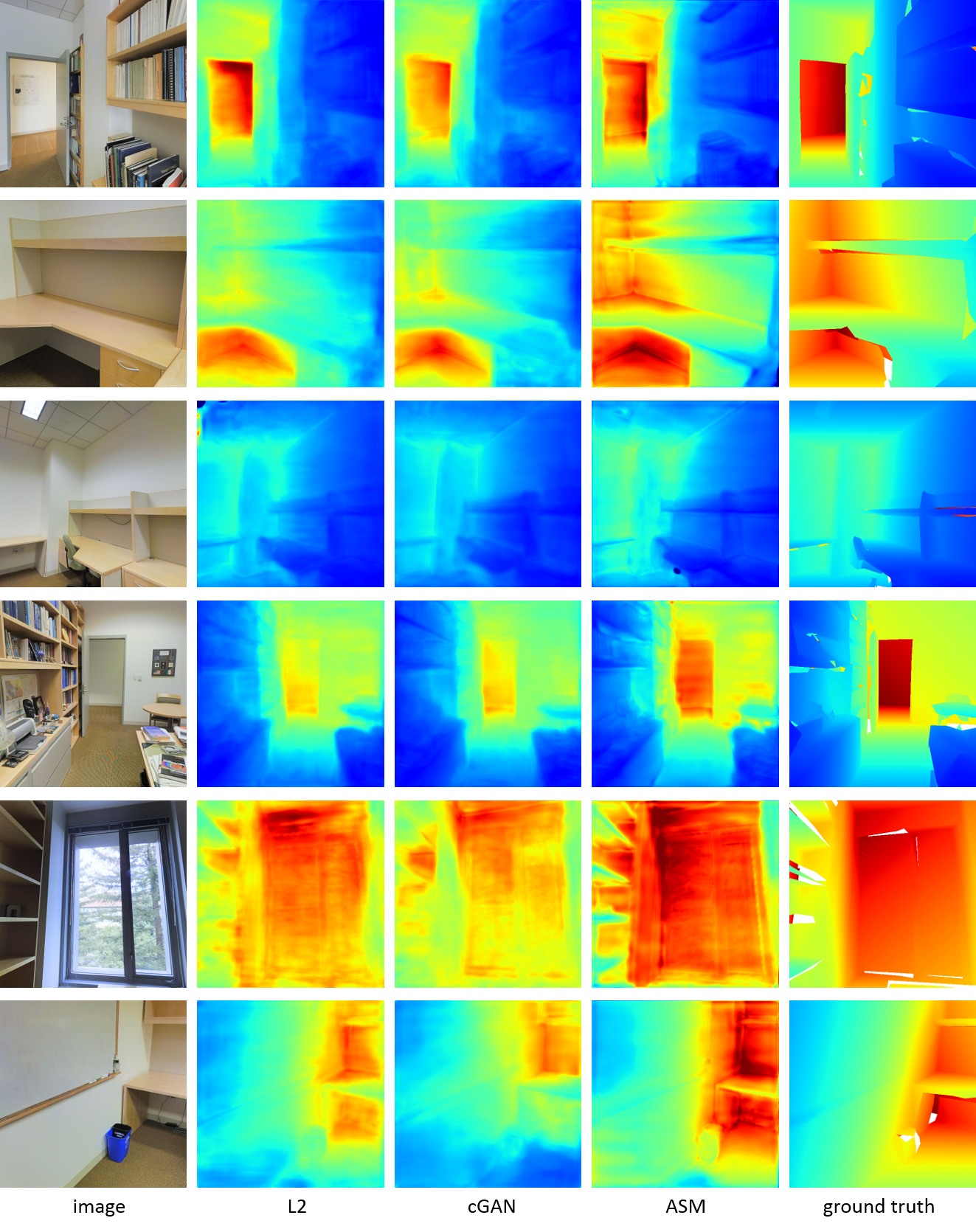}
    \caption{Visual quality comparison for depth estimation task on 2D-3D-Semantics~\cite{2D-3D-S}. From left to right: Image, Ground-truth, UNet~\cite{ronneberger2015u} trained using $L_2$ for depth and normalized $L_2$ for surface normal, cGAN and ASM. The models are jointly trained with depth and surface normal estimation.}
    \label{fig:stanford_joint}
\end{figure*}

\end{document}